\documentclass[preprint,number]{elsarticle}

\usepackage[title]{appendix}
\usepackage{amsmath}
\usepackage{amsthm}
\usepackage{amssymb}
\usepackage{bbm}
\usepackage{verbatim}
\usepackage{graphicx}
\usepackage{subfigure}
\usepackage{afterpage}
\usepackage{etoolbox}
\usepackage{footbib}
\usepackage{float}
\usepackage{algorithmic}
\usepackage[algoruled,vlined]{algorithm2e}
\usepackage{multirow}
\usepackage{rotating}

\numberwithin{equation}{section}
\numberwithin{figure}{section}
\numberwithin{table}{section}

\makeatletter
\renewcommand{\p@subfigure}{\thefigure}
\makeatother

\newtheorem{theorem}{Theorem}[section]

\newtheorem{lemma}[theorem]{Lemma}

\newcommand{\repeatable}[2]{\makeatletter \global\expandafter\def\csname repText@#1\endcsname {#2} \makeatother #2}
\newcommand{\repeatxt}[1]{\makeatletter \expandafter\csname repText@#1\endcsname \makeatother}

\newtoggle{citeInAbstract}
\toggletrue{citeInAbstract}

\newcommand{\usecrop}[2]
{
	\newlength{\cropwidth}
	\setlength{\cropwidth}{\the\textwidth}
	\addtolength{\cropwidth}{#1}
	\newlength{\cropheight}
	\setlength{\cropheight}{\the\textheight}
	\addtolength{\cropheight}{#2}
	\usepackage[width=\the\cropwidth,height=\the\cropheight,center]{crop}
}

\DeclareMathAlphabet{\mathpzc}{OT1}{pzc}{m}{it}

\DeclareMathOperator{\erf}{erf}

\newcommand{\transpose}[1]{#1^T}
\newcommand{\abs}[1]{\left | #1 \right |}

\newcommand{\norm}[1]{\left \| #1 \right \|}
\newcommand{\norminline}[1]{\| #1 \|}
\newcommand{\Rn}[1]{{\mathbb{R}^{#1}}}
\newcommand{\ve}{\varepsilon}
\newcommand{\inp}[2]{\left\langle #1 , #2 \right\rangle}
\newcommand{\inpinline}[2]{\langle #1 , #2 \rangle}

\usepackage{upgreek}
\usepackage{isomath}

\usecrop{1.69in}{1.5in}

\journal{Applied and Computational Harmonic Analysis}

\begin{document}

\begin{frontmatter}

\title{Diffusion Representations}
\author{Moshe Salhov}
\author{Amit Bermanis}
\author{Guy Wolf}
\author{Amir Averbuch\corref{maincor}}\ead{amir@math.tau.ac.il}

\address{School of Computer Science, Tel Aviv University, Tel Aviv 69978, Israel}

\cortext[maincor]{Corresponding author, Tel: +972-54-5694455, Fax: +972-3-6422020}

\begin{abstract}
Diffusion Maps framework is a kernel based method for manifold learning and data analysis that defines diffusion similarities by imposing a Markovian process on the given dataset. Analysis by this process uncovers the intrinsic geometric structures in the data. Recently, it was suggested to replace the standard kernel by a measure-based kernel that incorporates information about the density of the data. Thus, the manifold assumption is replaced by a more general measure-based assumption.

The measure-based diffusion kernel incorporates two separate independent representations. The first determines a measure that correlates with a density that represents normal behaviors and patterns in the data. The second consists of the analyzed multidimensional  data points.

In this paper, we present a representation framework for data analysis of datasets that is based on a closed-form decomposition of the measure-based kernel. The proposed representation preserves pairwise diffusion distances that    does not depend on the data size while being  invariant to scale. For a stationary data, no out-of-sample extension is needed for embedding newly arrived data points in the  representation space.  Several aspects of the presented methodology are demonstrated on analytically generated data.

\end{abstract}

\begin{keyword}
Manifold learning \sep kernel PCA \sep Diffusion Maps \sep diffusion distance \sep distance preservation
\end{keyword}

\end{frontmatter}

\section{Introduction}
%

Kernel methods constitute of a wide class of algorithms for non-parametric data analysis of massive high dimensional datasets. Typically, a limited set of underlying factors generates the high dimensional observable parameters via non-linear mappings. The non-parametric nature of these methods enables to uncover hidden structures in the data. These methods extend the well known Multi-Dimensional Scaling (MDS)~\cite{cox:MDS,kruskal:MDS} method. They are based on an affinity kernel construction that encapsulates the relations (distances, similarities or correlations) among multidimensional data points. Spectral analysis of this kernel provides an efficient representation of the data that simplifies its analysis. Methods such as Isomap~\cite{tenenbaum:Isomap}, LLE~\cite{roweis:LLE}, Laplacian eigenmaps~\cite{belkin:LaplacianEigenmaps}, Hessian eigenmaps~\cite{donoho:HessianEigenmaps} and local tangent space alignment~\cite{yang:LTSA,zhang:LTSA}, extend the MDS paradigm by assuming to satisfy the manifold assumption. Under this assumption, the data is assumed to be sampled from a low intrinsic dimensional manifold that captures the dependencies between the observable parameters. The corresponding spectral-based embedding performed by these methods preserves the geometry of the manifold that incorporates the underlying factors in the data.

The diffusion maps (DM) method~\cite{coifman:DM} is a kernel-based
method that defines diffusion similarities for data analyzes by imposing a Markovian  process over the dataset.
It defines a transition probability operator based on local
affinities between multidimensional data points. By spectral
decomposition of this operator, the data is embedded into a low
dimensional Euclidean space, where Euclidean distances represent the diffusion
distances in the ambient space.  When the data is sampled from a
low dimensional manifold, the diffusion paths follow the manifold
and the diffusion distances capture its geometry.

DM embedding was utilized for a wide variety of data and pattern
analysis techniques. For example it was used to improve audio
quality by suppressing transient
interference~\cite{talmon2013single}. It was utilized
in~\cite{Schclar2010111} for detecting and classifying moving vehicles.
Additionally, DM was applied to  scene
classification~\cite{SeamanticVisualWords2009}, gene expression
analysis~\cite{Cancer2007} and source
localization~\cite{TalmonSourceLocal2011}. Furthermore, the DM
method can be utilized for fusing different sources of
data~\cite{LafonDataFusion2006,Keller2010}.

DM embeddings in both the original
version~\cite{coifman:DM, lafon:PhD} and in the measure-based Gaussian
correlation (MGC) version~\cite{wolf:MGC, wolf:MGC-sampta}, are
obtained by the principal eigenvectors of the corresponding
diffusion operator. These eigenvectors represent the long-term
behavior of the diffusion process that captures its metastable
states~\cite{huisinga:metastability} as it converges to a unique
stationary distribution.

The MGC framework~\cite{wolf:MGC, wolf:MGC-sampta} enhances the DM
method by incorporating information about data distribution in addition to the local distances on which DM is based. This
distribution is modeled by a probability measure, which is assumed
to quantify the likelihood of data presence over the geometry of the
space. The measure and its support in this method replace the manifold assumption. Thus, the diffusion process is
accelerated in high density areas in the data rather than being depended
solely on the manifold geometry. As shown in~\cite{wolf:MGC}, the
compactness of the associated integral operator enables to achieve
dimensionality reduction by utilizing the DM framework.

This MGC construction consists of two independent data
points representations. The first represent the domain on which the measure is defined
or, equivalently, the support of the measure. The second represent  the
domain on which the MGC kernel function and the resulting diffusion
process are defined. These \emph{measure domain} and the \emph{analyzed
domain} may, in some cases, be identical, but separate sets can also
be considered by the MGC-based construction. The latter case utilizes a training dataset, which is used as the measure
domain to analyze any similar data that is used as the analyzed
domain. Furthermore, instead of using the collected data as an
analyzed domain, it can be designed as a dictionary or as a grid of
representative data points that capture the essential structure of
the MGC-based diffusion.

In general, kernel methods can find geometrical meaning in a given
data via the application of spectral decomposition. However, this representation
changes as additional data points are added to the given dataset. Furthermore,  the
required computational complexity, which is dictated by spectral
decomposition, is $O(n^3)$ that is not feasible for a very large
dataset.  For example, a segmentation of a medium size image of
$512 \times 512$ pixels requires a kernel matrix of size $2^{18}
\times 2^{18}$. The size of such  matrix necessitated about $270$
GByte of memory assuming double precision. Spectral
decomposition procedure applied to such a matrix is a
formidable slow task. Hence, there is a growing need to have more
computationally efficient methods that are practical for processing
large datasets.

Recently, a method to produce random Fourier features from a given data and a positive kernel was proposed in~\cite{Recht07}. The suggested method provides a random mapping of the data such that
the inner product of any two mapped points approximates the kernel function with high probability. This scheme utilizes Bochner's theorem~\cite{Bochner62} that says that any such kernel is a Fourier transform of a uniquely defined
probability measure. Later, this work was extended in \cite{Maji09,Vedaldi2012} to find explicit features for image classification.

In this paper, we focus on deriving a representation that preserves the diffusion distances between multidimensional data points based on the MGC framework~\cite{wolf:MGC, wolf:MGC-sampta}. This representation is applicable to process efficiently very large datasets by imposing a Markovian diffusion process to define and represent the non-linear relations between multidimensional data points. It provides a diffusion distance metric that correlates with the intrinsic geometry of the data. The suggested representation and its computational complexity cost per a given data point are invariant to the dataset size.

The paper is organized as follows: Section~\ref{sec:problem_formulation} presents the problem formulation. Section~\ref{sec:Explicit Form of the Diffusion Distance} provides an explicit formulation for the diffusion distance. The main results of this paper are established in Sections~\ref{sec:full_dm} and~\ref{sec:trunc_dm} that present the suggested low-dimensional embedding of the data and its characterization. Practical demonstration of the proposed methodology is given in Section~\ref{sec:examples}.

\section{Problem formulation and mathematical preliminaries}
\label{sec:problem_formulation}
Consider a big dataset $X \subseteq
\mathbb{R}^m$ such that for any practical purposes the size of $X$ is considered to be infinite. Without-loss-of-generality, we
assume that for all $x\in X$,    $\| x\| \leq 1$. Implementation of a kernel method, which uses a full spectral decomposition, becomes
impractical when the dataset size is big. Instead, we suggest to represent
the given dataset via the density of data points in it using the MGC kernel. In other
words, let $q: \mathbb{R}^m \rightarrow \left[0,1 \right]$  be the
density function of $X$. We aim to find an explicit embedding function denoted by $f_q : \mathbb{R}^m \rightarrow \mathbb{R}^k$,
where $k$ is the embedding dimension such that $m \ll k$. Each member
in $f_{q}$, which is  denoted by $f_{q}(x)$, $x\in X$, depends on the density function $q$.

DM provides a multiscale view of the data via a family of geometries that are referred to by \emph{diffusion geometries}. Each geometry is defined by both the associated diffusion metric and the diffusion time parameter $t$ that are linked by $d_\varepsilon^{(t)}:X\times X\to\Rn{}_+$ where $\ve$ is a localization parameter. The \emph{diffusion maps} are the associated functions $\Psi^{(t)}: X \to\Rn{k}$ that embed the data into Euclidean spaces, where the diffusion geometries are preserved such that $\norminline{\Psi^{(t)}(x)-\Psi^{(t)}(y)}\approx d_\varepsilon^{(t)}(x,y)$,  $x,y\in X$.

Given an accuracy requirement $\zeta>0$, we aim to design an embedding $f_{q}$  that preserves the diffusion geometry for $t=1$ such that for all $x,y \in X$,
\begin{equation}
\label{eq:appx_dm}
\abs{\norm{f_{q}(x)-f_{q}(y)} - d_\varepsilon^{(1)}(x,y)}\leq\zeta.
\end{equation}
We call the embedding $f_{q}$ the diffusion representation. From the requirement in Eq.~\eqref{eq:appx_dm}, the Euclidian distance between pairs of
representatives approximates the diffusion distance between the
corresponding data points over the density of these data points in the  MGC-based kernel when $t=1$.
If Eq.~\eqref{eq:appx_dm} holds, then $f_q$ preserves the diffusion geometry of the dataset in this sense.

 The rest of this section is dedicated to provide additional details regarding the diffusion geometries that utilize the MGC kernel.

\subsection{Diffusion geometries}
\label{subsec:DG}
A family of diffusion geometries of a measurable space $(X, \mu)$ with a measure $\mu$ is determined by imposing a Markov process over the space. Given a non-negative symmetric kernel function $k_\varepsilon:X\times X\to\Rn{}_+$, then an associated Markov process over the data via the stochastic kernel $p_\varepsilon:X\times X\to\Rn{}_+$ is
\begin{equation}
\label{eq:p}
p_\varepsilon(x,y)\triangleq k_\varepsilon(x,y)/\nu_\varepsilon(x),
\end{equation}
where $\nu_\varepsilon:X\to\Rn{}$ is the local volume function.  In a discrete setting, it is called the degree function. In a continues settings, the local volume function is defined by
\begin{equation}
\label{eq:deg_f}
\nu_\varepsilon(x)\triangleq\int_X k_\varepsilon(x,y)d\mu(y).
\end{equation}
The associated Markovian process over $X$ is defined via the conjugate operator of the integral operator $Pq(x) = \int_X p_\varepsilon(x,y)q(y)d\mu(y)$ that is denoted by $P^\ast$. Thus, for any initial probability distribution $q_0$ over $X$, $q_1=P^\ast q_0$ is the probability distribution over $X$ after a single time step. The probability distribution over $X$ after $t$ time steps is given by the $t$-th power of $P^\ast$. Specifically, if the initial probability measure is concentrated in a specific data point $x\in X$, i.e. $q_0 = \delta(x)$, then the probability distribution after $t$ time steps is $(P^\ast)^t\delta(x)$, denoted also by $p_\varepsilon^{(t)}(x,\cdot)$. Thus, $p_\varepsilon^{(t)}(x,y)$ is the probability that a random walker, which started his walk in $x\in X$, will end in $y\in X$ after $t$ time steps. Based on this, the $t$-time diffusion geometry is defined by the distances between probability distributions such that for all $x,y \in X$
\begin{equation}
\label{eq:pf_dif_dist}
d_\varepsilon^{(t)}(x,y)\triangleq\norm{p_\varepsilon^{(t)}(x,\cdot)-p_\varepsilon^{(t)}(y,\cdot)}_{L^2(\Rn{m})}.
\end{equation}
Equations~\ref{eq:appx_dm} and~\ref{eq:pf_dif_dist}  suggest that the embedding $f_{q}(x)$ approximately preserves (for $t=1$) the distance between probability distributions.
Such a family of geometries can be defined for any Markovian process and not necessarily for a diffusion process. It is proved in~\cite{coifman:DM} that under specific conditions, the defined Markovian process approximates the diffusion over a manifold from which the dataset $X$ is sampled. If the Markovian process is ergodic, then it has a unique probability distribution $\hat{\nu}_\varepsilon:X\to\Rn{}_+$, to which it converges independently of its initial distribution, namely, for any $y\in X$, $\hat{\nu}_\varepsilon(y) = \lim_{t\to\infty}p_\varepsilon^{(t)}(x,y)$, independently of $x$. This probability measure is an $L_1$ normalization of the local volume function (Eq.~\eqref{eq:deg_f}), i.e. $\hat{\nu}_\varepsilon(y) = {\nu}_\varepsilon(y)/\int_X\nu_\varepsilon(y)d\mu(y)$.

\subsection{Measure-based Gaussian Correlation kernel}
The measure-based Gaussian Correlation (MGC) kernel~\cite{wolf:MGC, wolf:MGC-sampta} defines the affinities between elements in $X$, which in this context, it is referred to as the \emph{analyzed domain} via their relations with the reference dataset $M$ that is referred to as the \emph{measure domain}. This framework enables to have a flexible representation of $X$, as long as $M$, which in some sense characterizes the data, is sufficiently large. For example, it was shown in~\cite{MGCunigrid} that in order to compute the MGC-based diffusion to a large, perhaps infinite dataset, it suffices to sample it to compute the associated DM of the sampled dataset $X$ and to extend it to the rest of the dataset via an out-of-sample extension procedure. The sampling rate, shown in~\cite{MGCunigrid},  depends on the density of $M$ and on the required accuracy. In this sense, $X$ is considered as a grid for the whole dataset.

Mathematically, for the analyzed domain $X\subset\Rn{m}$ and for the measure domain $M\subset\Rn{m}$ with a density function $q:M\to\Rn{}_+$ defined on the measure domain, the MGC kernel $k_\varepsilon:X \times X \to\Rn{}_+$ is defined as
\begin{equation}\label{eq:kernel}
k_\varepsilon(x,y)\triangleq\int_{\Rn{m}}g_m(r;x,\frac{\varepsilon}{2}I_m)g_m(r;y,\frac{\varepsilon}{2}I_m)q(r)dr,
\end{equation}
where  $I_m$ is an $m\times m$ unit matrix. For a fixed mean vector $\theta\in\Rn{m}$ and a covariance matrix $\Sigma\in\Rn{m\times m}$,  $~g_m(r;\theta,\Sigma):\Rn{m}\to\Rn{}_+$ is the normalized Gaussian function given by
\begin{equation}\label{eq:g_d_def}
g_m(r;\theta,\Sigma)\triangleq\frac{1}{(2\pi)^{m/2}\abs{\Sigma}^{1/2}}\exp\left\{-\frac{1}{2}(r-\theta)^T\Sigma^{-1}(r-\theta)\right\}.
\end{equation}
Since the MGC kernel in Eq.~\eqref{eq:kernel} is symmetric and positive, it can be utilized to establish a Markov process as was described in Section~\ref{subsec:DG}. The associated diffusion parameters from Eqs.~\eqref{eq:p} ,~\eqref{eq:deg_f} and~\eqref{eq:pf_dif_dist} are $p_\ve$, $\nu_\ve$ and~$d_\varepsilon^{(t)}$, respectively.

\section{Explicit forms for the diffusion distance and stationary distribution}
\label{sec:Explicit Form of the Diffusion Distance}

In general, the integral in Eq.~\eqref{eq:kernel} does not have an
explicit form. However, for our purposes, we adopt the Gaussian Mixture Model (GMM), which assumes that the density $q$ is a superposition of normal distributions. Under this assumption, $q$ takes the form
\begin{equation}
\label{eq:gaus_mix}
q(r)=\sum_{j=1}^n a_j g_m(r;\theta_j,\Sigma_j),\,\sum_{j=1}^n a_j=1,
\end{equation}
for appropriate mean vectors $\theta_j$ and covariance matrices $\Sigma_j$, $j=1,\ldots,n$, (see Eq.~\eqref{eq:g_d_def}). Estimating Eq.~\eqref{eq:gaus_mix} is a generally known problem that has been extensively investigated such as in~\cite{Day1969,mclachlan88} with many published implementations. Such an estimation enables to provide an explicit (closed form) representation of the diffusion geometry in Eq.~\eqref{eq:pf_dif_dist}.

First, a closed form for the inner product $\inpinline{p^{(1)}_\ve(x,\cdot)}{p^{(1)}_\ve(z,\cdot)}_{L^2(\Rn{k})}$, $x,z\in X$, is presented. This inner product closed form enables to get an explicit formulation for the first time step of the MGC-based DM distance $d_\ve^{(1)}(x,z)$. This formalism is established in Theorem~\ref{thm:gmm_bsc}, whose proof appears in~\ref{apx:ThmProof}.

\begin{theorem}
\label{thm:gmm_bsc}
Assume that the GMM assumption in Eq.~\eqref{eq:gaus_mix} holds and let $W_{x,z}$ be the inner product  $W_{x,z}\triangleq \inp{k_\ve(x,\cdot)}{k_\ve(z,\cdot)}_{L^2(\Rn{k})}$. Denote  $D_j \triangleq (\varepsilon^{-1}I_m+(\varepsilon I_m+4\Sigma_j)^{-1})^{-1}$ and $c_j(x)  \triangleq  D_j (\ve^{-1}x+(\ve I_m+4\Sigma_j)^{-1}(2\theta_j-x))$. Let $\tilde{\Sigma}_j$ be defined by
\begin{eqnarray}
\label{eq:sig_tild}
\tilde{\Sigma}_j & \triangleq & (\ve/2)I_m+\Sigma_j.
\end{eqnarray}
Then, for any $x,z\in X$, the kernel affinity $k_\ve(x,z)$, the stationary distribution $\nu_\ve(x)$ and the inner product $W_{x,z}$ have explicit forms given by
\begin{eqnarray}
k_\ve(x,z) & = & \sum_{j=1}^n a_j
g_m(x;\theta_j,\tilde{\Sigma}_j)g_m(z;c_j(x),D_j),
\end{eqnarray}
\begin{eqnarray}
\nu_\ve(x) & = & \sum_{j=1}^n a_j g_m(x;\theta_j,\tilde{\Sigma}_j)\label{eq:cf_of_neu}
\end{eqnarray}
and
\begin{eqnarray}
\label{eq:inner_int}
W_{x,z} = \sum_{j,i=1}^n a_j a_k g_m(x;\theta_j,\tilde{\Sigma}_j)g_m(z;\theta_i,\tilde{\Sigma}_i)g_m(c_j(x);c_i(z),D_j+D_i).
\end{eqnarray}
\end{theorem}

Combination of Theorem~\ref{thm:gmm_bsc} with Eqs.~\eqref{eq:p}, \eqref{eq:deg_f} and \eqref{eq:pf_dif_dist} provides a closed form solution for the first time step ($t=1$) diffusion metric to be
\begin{equation}
\label{eq:inner_prod_cf}
d_\varepsilon^{(1)}(x,y) =  \frac{W_{x,x}}{\nu_\ve(x)\nu_\ve(x)}+\frac{W_{z,z}}{\nu_\ve(z)\nu_\ve(z)}-\frac{2W_{x,z}}{\nu_\ve(x)\nu_\ve(z)}.
\end{equation}
Moreover, by combining Eqs.~\eqref{eq:g_d_def},~\eqref{eq:gaus_mix} and~\eqref{eq:cf_of_neu} we get $\norminline{\nu_\ve}_1 = 1$. Thus, the local volume function is equal to the stationary distribution (Eq.~\eqref{eq:deg_f}) of the MGC-based Markov process.

\section{Explicit Diffusion Maps of the analyzed domain }
\label{sec:full_dm}
The diffusion distance provides a relation between a pair of data
points in the analyzed domain. In this section, we find a
representation of any data point in the analyzed domain that
preserves the diffusion distance relation. We assume that the covariances in the GMM (Eq.~\eqref{eq:gaus_mix}) are all identical, namely, $\Sigma_j = \Sigma$, $j=1,\ldots,n$, and that the analyzed domain $X$ is a subset of the unit ball in $\Rn{m}$. The Taylor extension
\begin{equation}
\label{eq:exp}
e^{\gamma} = \sum_{i=0}^\infty \frac{\gamma^i}{i!},
\end{equation}
where ${\gamma}$ is a scalar is reformulated with ${\gamma}=-\frac{1}{2}\| x - y
\|^2$ to be
\begin{equation}
\label{eq:exp_norm} e^{ -\frac{1}{2}\| x - y \|^2_2} =e^{
-\frac{1}{2}\| x\|^2}   e^{-\frac{1}{2}\| y\|^2}  \sum_{i=0}^\infty
\frac{ \left(x^T y \right) ^i}{i!}.
\end{equation}
The exponent in Eq.~\eqref{eq:exp_norm} is formulated as an inner product
such that
\begin{equation}
\label{eq:inner_prod_cf_innprod}
e^{ -\frac{1}{2}\| x - y \|^2_2} =e^{  -\frac{1}{2}\| x\|^2}   e^{-\frac{1}{2}\| y\|^2} \phi \left( x \right)^T \phi \left( y \right),
\end{equation}
where $\phi \left( x \right)$ is a vector of infinite length whose first term ($i=1$) is $\left[\phi \left( x \right) \right]_1 = 1$ and for $i>1$ we have the generating function
$\left[\phi \left( x \right) \right]_{i+1} = \frac{1}{\sqrt{i+1 }} \left[ \phi \left( x \right)\right]_i \otimes x$ where $\otimes$ is the Kronecker product operator.
In the following we consider the vector $\left[ \phi \left( x \right)\right]_i $ as a single term of $\phi \left( x \right)$ that corresponds to the $i$th term from the relevant Taylor expansion.

For the diffusion representation, we reformulate the inner product
in Eq.~\eqref{eq:inner_int} to be a vector multiplication of two vectors.
The first depends only on $x$ and the second  depends only on $z$.
Let $h\left( x \right)$ be a $n$-dimensional vector with the entries
$\left[h\left( x \right) \right]_j \triangleq a_j
g_m(x;\theta_j,\tilde{\Sigma}) $ (Eq.~\eqref{eq:sig_tild}). Then, the inner product in
Eq.~\eqref{eq:inner_int} is given by the  matrix multiplication
\begin{equation}
\label{eq:inn_matform}
W_{x,z} = h\left( x \right)^T G h\left( z \right),
\end{equation}
where the $j$ and $k$ entry in the $n \times n$ matrix $G$ is given by
$g_m(c_j(x);c_k(z),D_j+D_k)$. The scalar $\left[ G \right]_{j,k}$
depends on both $x$ and $z$. From the assumption on the GMM
structure, we have $D_j= D_k = D$. Let $\hat{c}_j(x) \triangleq
\left(2D\right)^{-\frac{1}{2}}c_j(x)$ then   $\left[ G \right]_{j,k}$ can be
reformulated by
\begin{equation}
\label{eq:g_ele}
    \left[ G \right]_{j,k} = \frac{1}{(2\pi)^{m/2}\abs{2 D}^{1/2}}\exp\left\{- \frac{1}{2} (\| \hat{c}_j(x)-\hat{c}_k(z) \|^2) \right\}.
\end{equation}
By reformulating Eq.~\eqref{eq:g_ele} using the inner product in Eq.~\eqref{eq:inner_prod_cf_innprod}, we
get
\begin{equation}
    \left[ G \right]_{j,k} = \frac{1}{(2\pi)^{m/2}\abs{2 D}^{1/2}} e^{- \frac{1}{2} \| \hat{c}_j(x) \|^2}   e^{ -\frac{1}{2} \| \hat{c}_k(z)\|^2}   \phi \left( \hat{c}_j(x) \right)^T \phi \left( \hat{c}_k(z) \right).
\end{equation}
Hence, the matrix $G$ is the result from an  outer-product of the form
\begin{equation}
\label{eq:G_outer_prod}
     G =  \frac{(2\pi)^{-m/2}}{\abs{2 D}^{1/2}} F^T F,
  \end{equation}
 where
\begin{equation}
\label{eq:G_outer_prod}
     F =  \left[
   e^{ -\frac{1}{2} \| \hat{c}_1(z)\|^2} \phi \left( \hat{c}_1(z) \right)  ,\cdots,    e^{ -\frac{1}{2} \| \hat{c}_n(z)\|^2} \phi \left( \hat{c}_n(z) \right)  \right].
  \end{equation}
By substituting   the outer-product from Eq.~\eqref{eq:G_outer_prod} into Eq.~\eqref{eq:inn_matform}, we get
 \begin{equation}
\label{eq:inn_matform_vect}
 \frac{ h\left( x \right)^T G h\left( z \right) }{\nu_\ve(x) \nu_\ve(z)}= f\left( x \right)^T  f\left( z \right),
\end{equation}
where
 \begin{equation}
\label{eq:rep_diff_infinite}
 \transpose{f}\left( x \right) = \frac{(2\pi)^{-m/4}}{\abs{2 D}^{1/4} \nu_\ve(x) }  h\left( x \right)^T F^T
\end{equation}
and the infinite dimensional vector $ f\left( x \right)$ is the weighted sum of the  inner product in the exponent decompositions (Eq.~\eqref{eq:inner_prod_cf_innprod}).
\begin{lemma}
\label{lem:diff_rep}
The diffusion representations $f\left( x \right)$ and  $f\left( z \right)$ in Eq.~\eqref{eq:rep_diff_infinite} preserve the diffusion distance in Eq.~\eqref{eq:inner_prod_cf}.
\end{lemma}
\begin{proof}
Equation~\ref{eq:pf_dif_dist} together with Eqs.~\eqref{eq:inner_prod_cf},~\eqref{eq:inn_matform} and~\eqref{eq:inn_matform_vect} yield
\begin{equation}
\label{eq:diff_dist_preserv}
\|f\left( x \right) - f\left( z \right)\|_{\ell^2(\Rn{m})} =  \norm{p_\varepsilon(x,\cdot)-p_\varepsilon(z,\cdot)}_{L^2(\Rn{m})} .
\end{equation}
\end{proof}
From Lemma~\ref{lem:diff_rep} we conclude that $f\left( x \right)$ is an embedding that preserve the diffusion geometry (for $t=1$) where $\zeta=0$ in Eq.~\ref{eq:appx_dm}.
\section{Truncated diffusion representation}
\label{sec:trunc_dm}
The diffusion representation $f(x)$ in Lemma~\ref{lem:diff_rep} is
an infinite size vector. In this section, this vector is truncated and an error estimate for the resulted
truncated diffusion distance is provided. The infinite number of entries in
$f(x)$  originates from the Gaussian decomposition  multiplication in Eq.~\eqref{eq:inner_prod_cf_innprod}. Let $e^{\gamma}$ be
a Taylor expended function on the interval $0 \leq \gamma \leq \gamma_{max}$,
where $\gamma_{max}$ is an arbitrary upper limit of $\gamma$. Then, the Lagrange
remainder of the $l$th order Taylor (Maclaurin) expansion of $e^{\gamma}$
on this interval is given by
\begin{equation}
\label{eq:R_l}
R_l\left(\gamma \right) = \frac{e^b}{\left(l+1 \right)!}\gamma^{l+1},~~~0 \leq b \leq \gamma_{max}.
\end{equation}
Since $e^b \leq e^{\gamma_{max}}$ we have
 \begin{equation}
\label{eq:R_l_b}
R_l\left(\gamma \right) \leq \frac{e^{\gamma_{max}}}{\left(l+1 \right)!}\gamma^{l+1}.
\end{equation}
Let $f_l \left(x \right)$ be the $l$th order truncated version of $
f\left( x \right)$  from Eq~\eqref{eq:rep_diff_infinite} that is given by,
 \begin{equation}
\label{eq:rep_diff_trancated}
 \transpose{f}_l\left( x \right) = \frac{(2\pi)^{-m/4}}{\abs{2 D}^{1/4} \nu_\ve(x) }  h\left( x \right)^T \hat{F}^T
\end{equation}
where the matrix $\hat{F}$ is given by
 \begin{equation}
 \label{eq:trancated_frob}
\hat{F} \left( x \right) = \left[ \begin{array}{c}
 e^{- \frac{1}{2} \| \hat{c}_1(x) \|^2} \hat{\phi}_{l} \left( \hat{c}_1(x) \right) \\
        \vdots      \\
e^{- \frac{1}{2} \| \hat{c}_n(x) \|^2}   \hat{\phi}_{l} \left( \hat{c}_n(x) \right)    \end{array} \right]
 \end{equation}
and  $\hat{\phi}_{l} \left( x\right) $ is a vector that contains the first $l$  terms from ${\phi}_{l} \left( x\right) $.
Furthermore, let $\phi_{l+1} \left( x \right)$ be the vector
that contains the truncated terms
from the $(l+1)$th term in $\phi \left( x \right)$ to $\infty$. Then, the diffusion
distance error is given by
 \begin{equation}
\label{eq:rep_diff_dist}
 \|{f\left( x \right) - f_l\left( x \right)} \|^2_F = \frac{(2\pi)^{-m/2}}{\abs{2 D}^{1/2} \nu_\ve(x)^2 } \| F_l\left( x \right) ^T h\left( x \right) \|^2,
\end{equation}
where the matrix $F_l$ is given by
  \begin{equation}
 \label{eq:trancated_frob}
F_l \left( x \right) = \left[ \begin{array}{c}
 e^{- \frac{1}{2} \| \hat{c}_1(x) \|^2} \phi_{l+1} \left( \hat{c}_1(x) \right) \\
        \vdots      \\
e^{- \frac{1}{2} \| \hat{c}_n(x) \|^2}   \phi_{l+1} \left( \hat{c}_n(x) \right)    \end{array} \right].
 \end{equation}
The Frobenius matrix norm and the Euclidean vector norm are compatible~\cite{Meyer:2000} such that
 \begin{equation}
\label{eq:rep_diff_invc1}
 \| F_l\left( x \right) ^T h\left( x \right) \|^2_2 \leq  \| F_l\left( x \right) ^T  \|^2_F  \| h\left( x \right) \|^2_2.
\end{equation}
From Eqs.~\eqref{eq:g_d_def} and \eqref{eq:gaus_mix} we have $\sum_{j=1}^n a_j=1$ and $g_m(r;\theta_j,\Sigma_j) \leq 1$, then,
$\| h\left( x \right) \| \leq 1$. Furthermore,  from the Frobenius matrix
norm definition we have
 \begin{equation}
\label{eq:rep_bound_F}
 \| F_l\left( x \right) ^T \|^2_F = \sum_i e^{-  \| \hat{c}_i(x) \|^2} \phi_{l+1} \left( \hat{c}_i(x) \right)^T \phi_{l+1} \left( \hat{c}_i(x) \right).
\end{equation}
The product $\phi_{l+1} \left( \hat{c}_i(x) \right)^T \phi_{l+1}
\left( \hat{c}_i(x) \right)$ is the remainder of the Taylor series.
By using Eq.~\eqref{eq:R_l_b}, we get
 \begin{equation}
\label{eq:c_riminder}
    \phi_{l+1} \left( \hat{c}_i(x) \right)^T \phi_{l+1} \left( \hat{c}_i(x) \right) \leq \frac{e^{{\| \hat{c}_i(x)\|^2}}}{\left(l+1 \right)!}{\| \hat{c}_i(x)\|}^{2(l+1)},
\end{equation}
since in this case ${\gamma}_{max} = \| \hat{c}_i(x)\|^2$. By
substituting Eq.~\eqref{eq:c_riminder} into Eq.~\eqref{eq:rep_bound_F} and finding the argmax of $\| \hat{c}_i(x)\|$ over $x$,
we get
 \begin{equation}
\label{eq:rep_bound_F2}
 \| F_l\left( x \right) ^T \|^2_F  \leq  \sum_{i=1}^n \frac{\| \hat{c}_i(x_i^*)\|^{2(l+1)}}{\left(l+1 \right)!},
\end{equation}
where $x_i^*$ is the solution to the Trust Region Problem (TRP) of
the form
 \begin{equation}
\label{eq:trust_region1}
    x_i^* =  \max_{\| x\| \leq 1} \hat{c}_i(x)^T\hat{c}_i(x) =  \max_{\| x\| \leq 1} \| A x -b \|^2,
\end{equation}
where $A = 2^{-\frac{1}{2}}D^{\frac{1}{2}} \left(\ve^{-1}I_m - ( \ve
I_m+4\Sigma_j)^{-1}  \right)$ and $b = 2^{\frac{1}{2}} \theta_j D^{\frac{1}{2}}( \ve
I_m+4\Sigma_j)^{-1} $. TRP in Eq.~\eqref{eq:trust_region1}
is widely investigated in the literature. Important properties are given in \cite{GQT66,G81,MS83}. In~\ref{apx:Trust-Region}, we provide complementary theoretical details about this problem. In the following, we assume that the solution $x_i^*$ was found such that the maximal value of $\hat{c}_i(x)^T\hat{c}_i(x)$ is known.

Lemma~\ref{cor:5_1} provides an additional bound for $\| F_l\left( x \right) ^T \|^2_F $.
\begin{lemma}
\label{cor:5_1}
Assume that $F_l\left( x \right) $ is defined according to Eq.~\eqref{eq:trancated_frob}. Then, for all $x \in X$
\begin{equation}
\label{eq:trace_bound}
\| F_l\left( x \right) ^T \|^2_F \leq \sum_{i=1}^n \left [1  - e^{-\| \hat{c}_i(x_i^*) \|^2} \left( \sum_{j=0}^l  \frac{\| \hat{c}_i(x_i^*) \|^{2j}}{j!  }\right)\right].
\end{equation}
\begin{proof}
The inner product $ \phi_{l+1} \left( \hat{c}_i(x) \right)^T \phi_{l+1} \left( \hat{c}_i(x) \right)$ is the reminder of the Taylor series of $e^{\| \hat{c}_i(x) \|^2}$ with $l$ terms. Hence,
\begin{equation}
\label{eq:taylor_rem_2}
 \phi_{l+1} \left( \hat{c}_i(x) \right)^T \phi_{l+1} \left( \hat{c}_i(x) \right) = e^{\| \hat{c}_i(x) \|^2} -  \left( \sum_{j=0}^l  \frac{\| \hat{c}_i(x) \|^{2j}}{j!  }\right).
\end{equation}
Substituting Eq.~\eqref{eq:taylor_rem_2} into Eq.~\eqref{eq:rep_bound_F} yields,
  \begin{equation}
\label{eq:rep_bound_F4}
 \| F_l\left( x \right) ^T \|^2_F = \sum_i  \left( 1 -  e^{-  \| \hat{c}_i(x) \|^2} \sum_{j=0}^l  \frac{\| \hat{c}_i(x) \|^{2j}}{j!  }\right) =  \sum_i \varphi \left(\| \hat{c}_i(x) \|^2_2 \right) ,
\end{equation}
where $\varphi \left(z \right) \triangleq \left( 1 -  e^{- z} \sum_{j=0}^l  \frac{z^{j}}{j!  }\right) $.
The derivative $\frac{d\varphi }{dz}$ is given by
\begin{equation}
\label{eq:derivative}
 \frac{d\varphi }{dz} = e^{-z} \frac{z^l}{l!}.
\end{equation}
From Eq.~\eqref{eq:derivative} we get that $ \frac{d\varphi }{dz} \geq 0$ on the interval $z \in \left[ 0, \| \hat{c}_i(x_i^*)\|^2_2\right]$ ($x_i^*$ was defined in Eq.~\eqref{eq:rep_bound_F2}). Hence, $\varphi \left(z \right) $  is a monotonic function on the interval and  $\| F_l\left( x \right) ^T \|^2_F $ is bounded by the maximal values on the interval where  $\varphi \left(z \right) $ is defined by,
  \begin{equation}
\label{eq:rep_bound_F5}
 \| F_l\left( x \right) ^T \|^2_F \leq  \sum_i \varphi \left(\| \hat{c}_i(x_i^*) \|^2_2\right).
\end{equation}
\end{proof}
\end{lemma}
Let $G_{b_1} \left(l\right)  \triangleq  \sum_i \frac{\| \hat{c}_i(x_i^*)\|^{2(l+1)}}{\left(l+1 \right)!}$ be the bound from Eq.~\eqref{eq:rep_bound_F2} and let $G_{b_2} \left(l\right) \triangleq  \sum_i \varphi \left(\| \hat{c}_i(x_i^*) \|^2_2\right) $ be the bound from Eq.~\eqref{eq:trace_bound}. Furthermore, let $B_{l_{max}}$ be the bound on  $ \| F_l\left( x \right) ^T \|^2_F \leq B_{l_{max}}$ that takes into consideration the bounds in Eqs.~\eqref{eq:trace_bound} and \eqref{eq:rep_bound_F2} such that
\begin{equation}
\label{eq:agg_bound}
 B_{l_{max}} = min(G_{b1}\left(l\right),G_{b2}\left(l\right)).
\end{equation}

By examining the stationary distribution component $\nu_\ve(x)$ in the denominator of Eq.~\eqref{eq:rep_diff_dist}, we assume that a given a minimal threshold $\nu_{min}$ establishes that any $\nu_\ve(x)$, which is smaller than
$\nu_{min}$, is negligible. By substituting Eq.~\eqref{eq:agg_bound}  into Eqs.~\eqref{eq:rep_bound_F5} and $\nu_{min}$ into Eq.~\eqref{eq:rep_diff_dist} we get
 \begin{equation}
\label{eq:rep_diff_bound}
 \|{f\left( x \right) - f_l\left( x \right)} \|^2_2\leq  \eta \triangleq \frac{(2\pi)^{-m/2}}{\abs{2 D}^{1/2} \nu_{min}^2 } B_{l_{max}}.
\end{equation}
By imposing the  requirement in Eq.~\eqref{eq:appx_dm} for $l$ in the form
\begin{equation}
\label{eq:bound_const}
\|{f\left( x \right) - f_l\left( x \right)} \|^2_2\leq \frac{\zeta^2}{4}
\end{equation}
we find the minimal $l_{max}$  such that
 \begin{equation}
\label{eq:zeta_req}
 \frac{(2\pi)^{-m/2}}{\abs{2 D}^{1/2} \nu_{min}^2 } B_{l_{max}} \leq \frac{\zeta^2}{4}.
\end{equation}
The number  $l_{max}$ of Taylor terms of the Taylor expansion is derived by computing
Eq.~\eqref{eq:zeta_req} for an increasing integers from $1$ until the
requirement in Eq.~\eqref{eq:zeta_req} is satisfied. We expect that for a practical use, $100$
iterations are sufficient.

The error in the estimated embedding is bounded.
Hence,  the error in the estimated diffusion distance in Eq.~\eqref{eq:appx_dm} is also bounded as Lemma~\ref{cor:distancebound}  proposes.
\begin{lemma}
\label{cor:distancebound}
Let $x,y \in M$ be any two data points. Assume we have a properly trained GMM and the number of Taylor terms $l$ is designed according to Eq.~\eqref{eq:zeta_req}. Then, the error in the diffusion distance (Eq.~\eqref{eq:appx_dm})  is bounded by $\zeta$.

\begin{proof}
Let $f(x)$ and $f(y)$ be the corresponding embedding of $x$ and $y$, respectively. Furthermore, let $f_l(x)$ and $f_l(y)$ be the corresponding truncated embedding of $x$ and $y$, respectively.
From Eq.~\eqref{eq:bound_const} we have $\|{f\left( x \right) - f_l\left( x \right)} \| \leq \frac{\zeta}{2}$ and $\|{f\left( y \right) - f_l\left( y \right)} \| \leq \frac{\zeta}{2}$.
Therefore, from the triangle inequality
\begin{equation*}
\begin{array}{ll}
\|{f\left( x \right) - f\left( y \right)} \| \leq & \|{f\left( x \right) - f_l\left( x \right)} \| + \|{f_l\left( x \right) - f_l\left( y \right)} \|+\|{f_l\left( y \right) - f\left( y \right)} \|, \\
 & \leq  \|{f_l\left( x \right) - f_l\left( y \right)} \|+ \zeta,
\end{array}
\end{equation*}
\begin{equation*}
\begin{array}{ll}
\|{f_l\left( x \right) - f_l\left( y \right)} \| \leq & \|{f\left( x \right) - f_l\left( x \right)} \| + \|{f\left( x \right) - f\left( y \right)} \|+\|{f_l\left( y \right) - f\left( y \right)} \|, \\
 & \leq  \|{f\left( x \right) - f\left( y \right)} \|+ \zeta.
\end{array}
\end{equation*}
Hance, $ \|{f\left( x \right) - f\left( y \right)} \| - \zeta \leq \|{f_l\left( x \right) - f_l\left( y \right)} \| \leq \|{f\left( x \right) - f\left( y \right)} \| + \zeta $.
\end{proof}
\end{lemma}

\subsection{The computational complexity to compute the diffusion representation}
\label{sec:computational_cost_analysis}
Assume  that GMM-based $q(r)$ and $l_{max}$ are designed according to Eq.~\eqref{eq:zeta_req}. Assume that $l_{max}$ is the length of the representation that
is dictated by the dimension of the measurement $m$ such that  $k = \sum_{i=1}^{l_{max}} m^{i-1}+1$. The computation of $\phi \left(x \right)$  in Lemma.~\ref{lem:diff_rep} has a computational complexity of
$O \left(  m^{l_{max}-1 }n \right)$ due to the multiplication $h\left(x\right)^T F\left(x\right)$ and due to the truncation of $\phi \left(x \right)$ that has $l_{max}$ Taylor terms.
We also assume that $l_{max} > 2$. Hence, the computational complexity for the computation of the terms in the vector $h\left(x\right)$ is negligible compared to the computational complexity of $\phi \left( x \right)$.

\section{Experimental Results}
\label{sec:examples}
This section presents several examples that demonstrate the principles  of the MGC-based closed-form embedding. The first example presents an analysis of a density function for which the stationary distribution is analytically known. The closed-form stationary distribution in this case is compared to the analytical stationary distribution. The second example compares the analytical distance  computation with the corresponding distance that is computed via the closed-form embedding defined in Eq.~\eqref{eq:rep_diff_infinite}.  

\subsection{Example I: A data analysis  with an analytically known stationary distribution function}
Let the density function $q(r) \in \mathbbm{R}^2$ includes two flat squares one above the other with probability $\frac{1}{5}$ to draw samples from the lower square and $\frac{4}{5}$ to draw samples from the upper square. In other words,
\begin{equation}
\label{eq:given_q}
 q\left(r \right) = \frac{1}{5}\chi_{ \left[0, 1 \right] \times \left[0, 1 \right] }(r) + \frac{4}{5}\chi_{ \left[3, 4 \right] \times \left[3, 4 \right] } (r)
\end{equation}
where $\chi_{\left[a, b \right] \times \left[c, d \right] }$ is the indicator function for the square $abcd$. Equation.~\eqref{eq:deg_f} formulates the stationary distribution computation. Given $\varepsilon=1$, the integration in Eq.~\eqref{eq:deg_f} is analytically solved by
\begin{equation}
\label{eq:nu_simple}
\nu \left(x_1,x_2\right) =  0.2 H(0,1,x_1,x_2) + 0.8  H(3,4,x_1,x_2),
\end{equation}
where $H(a,b,x_1,x_2)$ is  given by
\begin{align*}
 H(a,b,x_1,x_2)   & \triangleq \nonumber\\
& \frac{1}{4} \left ( \erf(b-x_1) - \erf(a-x_1) \right ) \left ( \erf(b-x_2)  - \erf(a-x_2) \right ),
\end{align*}
and $\erf(x)$ is the Gauss error function.
%

Given a trained GMM, the stationary distribution is computed by using Eq.~\eqref{eq:cf_of_neu}. First, $2000$ data points were randomly selected from the distribution in Eq.~\eqref{eq:given_q}. Then, a $48 \times 48 $ grid was constructed to compute the stationary distribution via  Eq.~\eqref{eq:cf_of_neu} based on the analytical solution in Eq.~\eqref{eq:nu_simple}.

\begin{figure}[H]
\centering
\subfigure[ Analytical stationary distribution]{\label{fig:discr_ana} \includegraphics[width=0.56\textwidth] {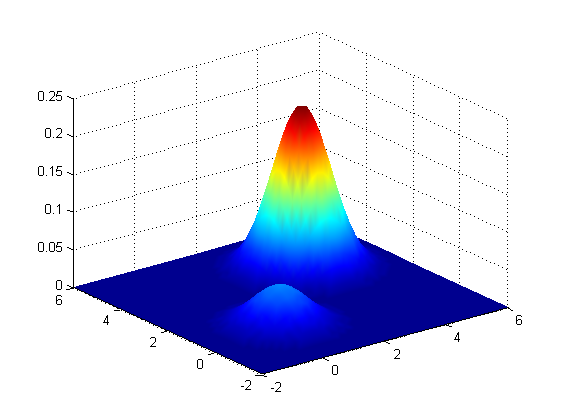}}
\subfigure[Closed-form stationary distribution] {\label{fig:discr_cf} \includegraphics[width=0.56\textwidth] {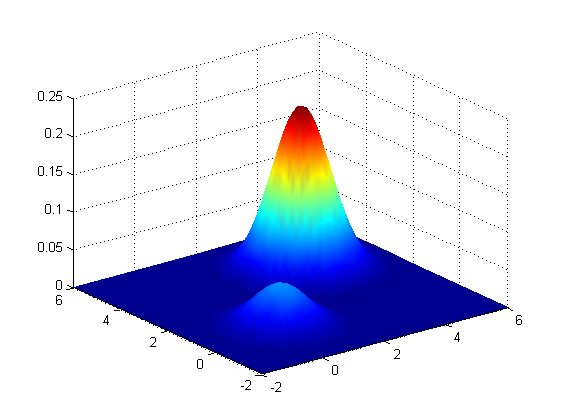}}
\subfigure[The error]{\label{fig:discr_diff} \includegraphics[width=0.56\textwidth] {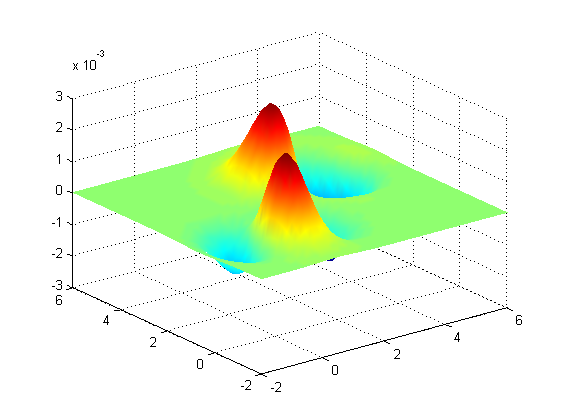}}
\caption{The difference (error) between the closed-form stationary distribution (Eq.~\eqref{eq:cf_of_neu}) and the analytical stationary distribution (Eq.~\eqref{eq:nu_simple}). (a) Analytical stationary distribution. (b) Closed-form stationary distribution. (c) The error between (a) and (b). }
\label{fig:exa_1_compare_v}
\end{figure}

Figure~\ref{fig:discr_cf} presents the stationary distribution (Eq.~\eqref{eq:cf_of_neu}) with a minor distortion compared to the analytical distribution  (Eq.~\eqref{eq:nu_simple}) presented in Fig.~\ref{fig:discr_ana}. The error, which is presented in Fig.~\ref{fig:discr_diff}, is a result from the GMM-based training on a small set of data points. The difference between the closed form stationary distribution and the analytical stationary distribution is in the order of $1\%$ and it is the result from the GMM training error.  In the following, we assume that the GMM-based training error is negligible.

\subsection{Example II: Estimation of the diffusion distance error originated from the truncated representation (Eq.~\eqref{eq:rep_diff_trancated})}
In this section, we show the diffusion distance error from the proposed truncated representation (Eq.~\eqref{eq:rep_diff_trancated}) of a specific data. Figures~\ref{fig:exa_2_hist} and \ref{fig:exa_2_compare_bound} compare between the diffusion distance in Eq.~\eqref{eq:inner_prod_cf} and the corresponding distance computed by Eq.~\eqref{eq:diff_dist_preserv}, respectively, using the truncated representations $f_l (x)$ and $f_l(z)$ computed by
\begin{equation}
\label{eq:delta_error}
\delta =  \left| \left( \frac{W_{x,x}}{\nu_\ve(x)\nu_\ve(x)}+\frac{W_{z,z}}{\nu_\ve(z)\nu_\ve(z)}-\frac{2W_{x,z}}{\nu_\ve(x)\nu_\ve(z)} \right) -  \|f_l (x) - f_l(z) \| \right|^2.
\end{equation}
A data of $3000$ data points was randomly selected from the distribution  in Eq.~\eqref{eq:given_q}. Furthermore, the bound on the worst-case diffusion distance error $\eta$ is computed using Eq.~\eqref{eq:rep_diff_bound} for the corresponding $l_{max}$ and $\varepsilon$ where $\nu_{min} = 10^{-3}$.

Figure~\ref{fig:exa_2_hist} displays the normalized  histogram of $\delta$ (Eq.~\eqref{eq:delta_error}) for different values of $\varepsilon \in [2^{-5},2^{5}]$ and different values of $l_{max} \in [1,14]$.  Figures~\ref{fig:HvsEps} and \ref{fig:HvsL} display a reduction of the distance error $\delta$ as a function of $\varepsilon$ and $l_{max}$, respectively. As $\varepsilon$ increases, $\delta$ is reduced in Fig.~\ref{fig:HvsEps}. Furthermore, as  the number of Taylor terms increases, $\delta$ is reduced in Fig.~\ref{fig:HvsL}.
\begin{figure}[H]
\centering
\subfigure[ Histogram of $\delta$ vs. $\varepsilon $ and  $l_{max}=3$]{\label{fig:HvsEps} \includegraphics[width=0.66\textwidth] {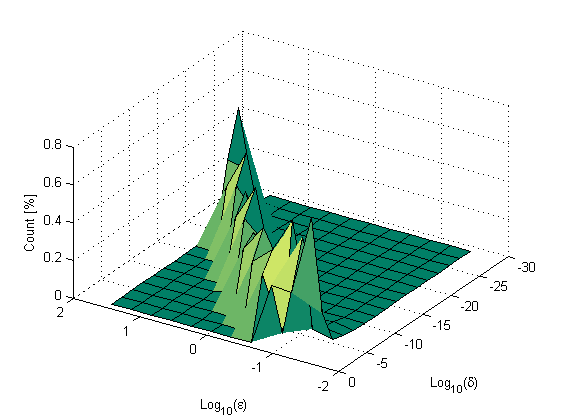}}
\subfigure[ Histogram of $\delta$ vs.  $l_{max} $ and  $\varepsilon=1$] {\label{fig:HvsL} \includegraphics[width=0.66\textwidth] {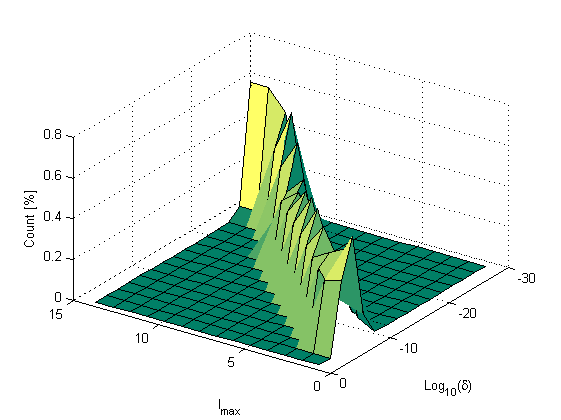}}
\caption{The histogram of the distance error $\delta$ as a function of $l_{max}$ and $\varepsilon$. (a) As a function of $\varepsilon \in [2^{-5},2^{5}]$ and fixed $l_{max}=3$. (b) As a function of $l_{max} \in [1,14]$ and fixed $\varepsilon = 1$}
\label{fig:exa_2_hist}
\end{figure}

Figure~\ref{fig:exa_2_compare_bound} compares the bound $\eta$ to the  worst-case of $\delta$ from the  $3000$ randomized data points as a function of $\varepsilon \in [2^{-5},2^{5}]$ and as a function of  $l_{max} \in [1,14]$. Figure.~\ref{fig:delta} compares between the worst-case of the diffusion distance error $\delta_{wc}$ on the dataset and the corresponding bound $\eta$ as a function of $\varepsilon$ and as a function of $l_{max}$.
\begin{figure}[H]
\centering
\subfigure[ $\delta_{wc}$ Vs. $l_{max}$ and $\varepsilon$  ]{\label{fig:delta} \includegraphics[width=0.48\textwidth] {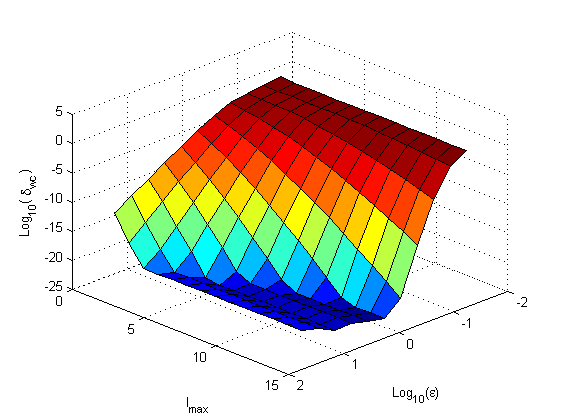}}
\subfigure[$\eta$ Vs. $l_{max}$ and $\varepsilon$ ] {\label{fig:eta} \includegraphics[width=0.48\textwidth] {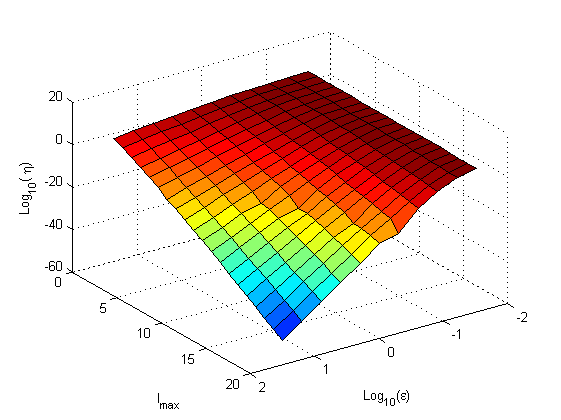}}
\subfigure[ $\eta - \delta_{wc}$  Vs. $l_{max}$ and $\varepsilon$ ]{\label{fig:etaMdelta} \includegraphics[width=0.48\textwidth] {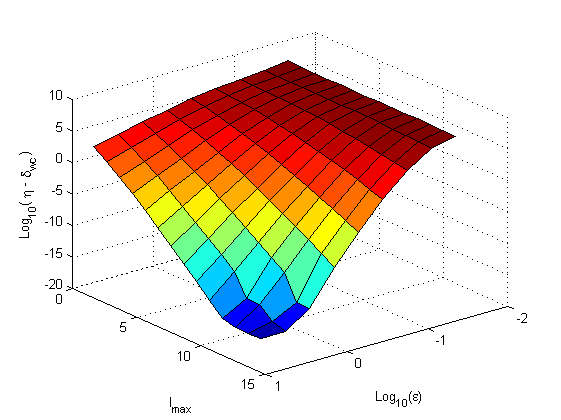}}
\subfigure[ $\eta$,  $\delta_{wc}$  Vs. $l_{max}$ and $\varepsilon$]{\label{fig:eta_and_delta} \includegraphics[width=0.48\textwidth] {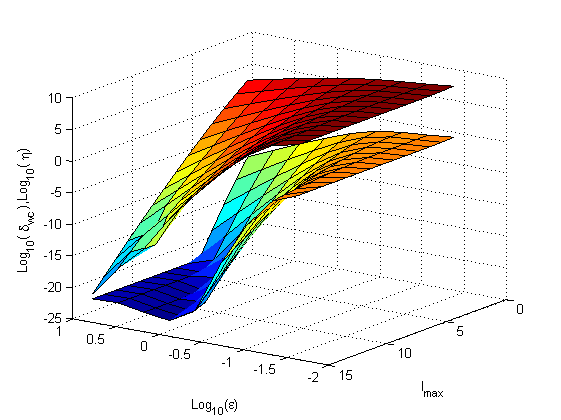}}
\caption{Comparison between the worst-case of the diffusion distance from the data  error denoted by $\delta_{wc}$ and the corresponding bound $\eta$ as a function of $\varepsilon$ and as a function of $l_{max}$. (a)  $\delta_{wc}$ as a function of $\varepsilon$ and  $l_{max}$. (b) $\eta$ as a function of $\varepsilon$ and  $l_{max}$. (c) The difference $\eta - \delta_{wc}$ as a function of $\varepsilon$ and  $l_{max}$. (d) The bound $\eta$ and the worst-case error $\delta_{wc}$ as a function of $\varepsilon$ and as a function of $l_{max}$}
\label{fig:exa_2_compare_bound}
\end{figure}

The worst-case error $\delta_{wc}$ in Fig.~\ref{fig:delta}  decreases as $l_{max}$ and $\varepsilon$ increase. Similarly, the corresponding bound in Fig.~{\ref{fig:eta} decreases as $l_{max}$ and $\varepsilon$ increase. The difference between $\delta_{wc}$ and the bound $\eta$ is shown in Figs.~\ref{fig:etaMdelta} and \ref{fig:eta_and_delta}, respectively.  $\eta$ bounds  $\delta_{wc}$ from above.

\section{Conclusions}
The presented methodology provides an alternative to a non-parametric kernel method approach for obtaining data representations via spectral decompositions of a big kernel operator or matrix with finite settings. The presentation of our approach is based on the MGC diffusion kernel~\cite{wolf:MGC} and on the resulting measure-based DM embedding obtained by its decomposition. We showed that when the underlying measure is modeled by a GMM, an equivalent embedding, which preserves the diffusion geometry of the data, can be computed without the need to decompose  the full kernel. Furthermore, this computation does not depend on the size of the dataset, therefore, it is suitable for modern scenarios where large amounts of data are available for analysis.

The proposed data representation is achieved by finding and decomposing an explicit form of the measure-based diffusion distance in terms of GMM. The decomposition is done by utilizing a Taylor decomposition and by an associated bound for the number of Taylor terms that was formulated for a given measure-based diffusion distance error tolerance requirement. The resulted representation length grows exponentially with the number of Taylor terms, hence, for a practical use, the error tolerance requirement  should consider the data dimensionality.

\subsection{Future work}
Future work will focus on extending the proposed methodology to a more general settings in the following ways below to reduce the representation dimension and to generalize the explicit form for additional kernels.

The utilized generating function for the inner products decomposition is based on the Kronecker product and hence incorporates its structure. This structure maybe utilized to reduce the representation length by using a variant of random projection. Thus, the dimensionality of the representation can reduced.

The proposed methodology is valid for additional Gaussian based kernels such as the standard DM kernel. The formulation of an explicit form for this kernel is straightforward for the above analysis. Hence, similar considerations may lead to an explicit form for the diffusion distance in this case. This research direction may lead to more general explicit forms that maybe used to represent datasets.

We provided an explicit form for the inner product associated with the measure-based diffusion distance, therefore, we are able to explicitly and efficiently compute a matrix of pair-wise distances that consider the entire dataset. This matrix can be decomposed via the application of MDS and the proposed representation can be projected to the resulted MDS space that may has a low dimension due to the involved spectral decomposition.

The explicit  resulted representation form can be used to find the most distant data points in terms of the diffusion distance by solving a proper optimization problem. The resulted data points can be used to form a dictionary or to provide data points samples for any general use.

The proposed explicit map and the explicit diffusion distance are differentiable, hence, both may be used to characterize data points in terms of change rate in the diffusion maps as a function of a given specific feature. This characterization may be helpful in identifying both the anomality and the associated features that are the source for anomality occurrence  .

The proposed representation corresponds to $t=1$. An additional interesting research direction is to generalize the analysis to any $t>1$ to allow a time diffusion multiscale  analysis of datasets.

\section*{Acknowledgment}
\noindent This research was partially supported by the Israeli Ministry of Science \& Technology (Grants No. 3-9096, 3-10898), US-Israel Binational Science Foundation (BSF 2012282), Blavatnik Computer Science Research Fund and Blavatink ICRC Funds.

\appendix
\section{Proof of Theorem~\ref{thm:gmm_bsc}}
\label{apx:ThmProof}
First, we introduce some basic Gaussian distribution related equalities. Denote  $\Lambda_{ij}\triangleq (\Sigma_i^{-1}+\Sigma_j^{-1})^{-1},$
$\psi_{ij}\triangleq \Lambda_{ij}(\Sigma_i^{-1}\theta_i+\Sigma_j^{-1}\theta_j)$ and
$\Psi_{ij}\triangleq \Sigma_i+\Sigma_j.$ In the following, we will use the  normal distribution identity
\begin{equation}
\label{eq:g_d_prod}
g_m(y;\theta_i,\Sigma_i)\cdot g_m(y;\theta_j,\Sigma_j)=g_m(y;\psi_{ij},\Lambda_{ij})\cdot g_m(\theta_i;\theta_j,\Psi_{ij}).
\end{equation}
Particularly,
\begin{equation}
\label{eq:g_d_sqr}
g_{m}^{2}(y;\theta,\Sigma)=(2\pi)^{-d/2}\abs{2\Sigma}^{-1/2}g_m(y;\theta,(1/2)\Sigma).
\end{equation}
By combining Eq.~\eqref{eq:g_d_prod} and the fact that the Gaussian function in Eq.~\eqref{eq:g_d_def} is normalized to a $L^2$ unit norm, we have
\begin{equation}
\label{eq:g_d_correlation} \int g_m(y;\theta_i,\Sigma_i)\cdot
g_m(y;\theta_j,\Sigma_j)dy = g_m(\theta_i;\theta_j,\Psi_{ij}).
\end{equation}
Another useful Gaussian relation is due to rescaling of the covariance matrix $\Sigma$ and
the mean vector $\theta$ in Eq.~\eqref{eq:g_d_def} such that
\begin{equation}
\label{eq:g_d_avg}
g_m \left(\frac{x+y}{2};\theta,\Sigma\right)=4^{d/2}g_m\left(y+x;2\theta,4\Sigma\right).
\end{equation}
Now, we are ready for the proof of Theorem~\ref{thm:gmm_bsc}.

\begin{proof}
An equivalent representation of the MGC kernel in Eq.~\eqref{eq:kernel}, as was proved in~\cite{wolf:MGC}, is
\begin{equation}\label{eq:kernel_reform}
k_\varepsilon\left(x,y\right) =  g_m\left(y;x,\ve I_m\right) \int g_m\left(r;\frac{x+y}{2},\frac{\ve}{4}I_m\right) q(r) dr.
\end{equation}
By substituting the GMM model (Eq.~\eqref{eq:gaus_mix}) into Eq.~\eqref{eq:kernel_reform}, we
get
\begin{equation}
\label{eq:kernel_gmm} k_\ve(x,y)  = g_m(y;x,\ve I_m)\sum_{j=1}^n
a_j\int g_m\left(r;\frac{x+y}{2},\frac{\ve}{4}I_m\right)\cdot
g_m(r;\theta_j,\Sigma_j)dr.
\end{equation}
Thus, from Eq.~\eqref{eq:g_d_correlation}, we get
\begin{equation}
k_\ve(x,y) = g_m(y;x,\ve I_m)\sum_{j=1}^n a_j g_m\left(\frac{x+y}{2};\theta_j,\frac{\ve}{4}I_m+\Sigma_j\right).
\end{equation}
Due to Eqs.~\ref{eq:g_d_prod} and \ref{eq:g_d_avg}, we get
\begin{eqnarray*}
g_m(y;x,\ve I_m) g_m\left(\frac{x+y}{2};\theta_j,\frac{\ve}{4}I_m+\Sigma_j\right)& = & 4^{d/2}g_m(y;x,\ve I_m) g_m(y;2\theta_j-x,\ve I_m+4\Sigma_j)\\
& = & 4^{d/2}g_m(x;2\theta_j-x,2\varepsilon I_m+4\Sigma_j)g_m(y;c_j(x),D_j)\\
& = & g_m(x;\theta_j,(\ve/2) I_m+\Sigma_j)g_m(y;c_j(x),D_j)
\end{eqnarray*}
where $D_j\triangleq (\varepsilon^{-1}I_m+(\varepsilon I_m+4\Sigma_j)^{-1})^{-1} = \ve(2\ve I_m+4\Sigma_j)^{-1}(\ve I_m+4\Sigma_j)$ and   $c_j(x)\triangleq D_j (\ve^{-1}x+(\ve I_m+4\Sigma_j)^{-1}(2\theta_j-x))$.
Thus, for $\tilde{\Sigma}_j\triangleq (\ve/2)I_m+\Sigma_j$,
we get
\begin{equation}\label{eq:kernel_ele_cf}
k_\ve(x,y)=\sum_{j=1}^n a_j
g_m(x;\theta_j,\tilde{\Sigma}_j)g_m(y;c_j(x),D_j).
\end{equation}
Thus, an immediate consequence from Eq.~\eqref{eq:deg_f} is $\nu_\ve(x)  =  \sum_{j=1}^n a_j g_m(x;\theta_j,\tilde{\Sigma}_j)$. Finally, since
\begin{equation}
\label{eq:inner_int_3}
\begin{array}{ll}
&k_\ve(x,y)\cdot k_\ve(z,y)=  \\ \nonumber
& \sum_{j,i=1}^n g_m(x;\theta_j,\tilde{\Sigma}_j)g_m(y;c_j(x),D_j)g_m(z;\theta_i,\tilde{\Sigma}_i)g_m(y;c_j(z),D_i),
\end{array}
\end{equation}
then by utilizing Eq.~\eqref{eq:g_d_correlation}, we get
\begin{eqnarray*}
\inp{k_\ve(x,\cdot)}{k_\ve(z,\cdot)}_{L^2(\Rn{d})} & = &\int k_\ve(x,y)k_\ve(z,y)dy\\
& = &\sum_{j,i=1}^n a_j a_i g_m(x;\theta_j,\tilde{\Sigma}_j)g_m(z;\theta_i,\tilde{\Sigma}_i)g_m(c_j(x);c_k(z),D_j+D_i).
\end{eqnarray*}
\end{proof}

\section{The Trust Region optimization problem (TRP)}
\label{apx:Trust-Region}
The trust region problem (or subproblem) is an optimization problem of the form
\begin{equation}
\label{eq:trust_region}
 {w}_{opt} = \arg \min_{\parallel w
\parallel_2 \leq \rho} \quad \left\{ w^* B  w +2w^* u \right\},
\end{equation}
where $B \in R^{m \times m}$ is a symmetric matrix and $u \in
R^m$. In the diffusion representation case, we have
\begin{equation}
\label{eq:trust_region2}
 w_{opt} = \arg \max_{\parallel w
\parallel_2 \leq \rho}
\parallel A  w + b
\parallel_2^2 = \arg  \min_{\parallel w
\parallel_2 \leq \rho} \left\{ w^* \left(-A^T A \right) w -2w^* A^T b - b^T b \right\}.
\end{equation}
Define the Lagrange function
\begin{equation}
\label{eq:trust_region_lagrange}
 L \left(w,\tau \right) = w^*\left(-A^T A \right)
 w -2w^T A^T b - b^T b +\tau \left(w^* w - \rho^2 \right).
\end{equation}
Let $B=A^T A$ and $u=A^T b $. A necessary and
sufficient conditions for $w$ to be an optimal solution for
the TRP is that there exists $\tau_{opt} \geq 0$
such that the following are satisfied \cite{SS82}
\begin{description}
\item $w_{opt}^T w_{opt} \leq \rho^2$ (Feasibility),
 \item $\left( \tau_{opt} I - B \right) w_{opt} = u$
 (Stationarity),
  \item $\tau_{opt} \left(w_{opt}^* w_{opt} - \rho^2 \right)=0$
 (Complementary Slackness),
\item $ \tau_{opt} I - B \succeq 0$ (2nd order necessary
conditions).
\end{description}
Furthermore, when $ \tau_{opt} I - B > 0$, then the optimal
solution $w_{opt}$ is unique.  Depending on the values of
$A$, $b$ and $\rho$, different ways of solving the TRP need to be considered. Different cases may
occur and are referred in the literature as the easy case and
the hard case. The hard case or near hard case is what causes
numerical difficulties in solving the TRP.
For the easy case, the optimization problem in Eq.~\eqref{eq:trust_region} has a single solution. However, for the hard case or near hard case, this optimization problem had multiple solutions that maximizes the optimization objective. For more details see \cite{Yuan2015}.


\bibliographystyle{plain}
\bibliography{MGC_cf}
\end{document}